\documentclass{article}

\usepackage[utf8]{inputenc} 
\usepackage[T1]{fontenc}    
\usepackage{hyperref}       
\usepackage{booktabs}       
\usepackage{amsfonts}       
\usepackage{nicefrac}       
\usepackage{microtype}      
\usepackage{amsmath}
\usepackage{enumitem}
\usepackage{amsthm}
\usepackage{graphicx}
\usepackage{subcaption}
\usepackage{color, colortbl}
\usepackage{wrapfig}
\usepackage{makecell}

\newtheorem{theorem}{Theorem}

\definecolor{Gray}{gray}{0.9}

\newcommand{\D}{\mathcal{D}}
\newcommand{\W}{\mathcal{W}}
\newcommand{\B}{\mathcal{B}}
\newcommand{\bx}{\mathbf{x}}
\newcommand{\by}{\mathbf{y}}
\newcommand{\bz}{\mathbf{z}}

\newcommand{\bb}{\mathbf{b}}
\newcommand{\bw}{\mathbf{w}}
\newcommand{\bW}{\mathbf{W}}
\newcommand{\bX}{\mathbf{X}}

\newcommand{\bs}{\mathbf{s}}
\newcommand{\thefield}{\mathbb{Z}/Q\mathbb{Z}}
\DeclareMathOperator*{\argmax}{argmax}

\usepackage[accepted]{icml2020}

\icmltitlerunning{Privacy-Preserving Multi-Party Contextual Bandits}

\begin{document}

\twocolumn[
\icmltitle{Privacy-Preserving Multi-Party Contextual Bandits}

\icmlsetsymbol{equal}{*}

\begin{icmlauthorlist}
\icmlauthor{Awni Hannun}{equal,fb}
\icmlauthor{Brian Knott}{equal,fb}
\icmlauthor{Shubho Sengupta}{fb}
\icmlauthor{Laurens van der Maaten}{fb}
\end{icmlauthorlist}

\icmlaffiliation{fb}{Facebook AI Research}

\icmlcorrespondingauthor{Awni Hannun}{awni@fb.com}

\icmlkeywords{differential privacy, multiparty computation, bandits}

\vskip 0.3in
]

\printAffiliationsAndNotice{\icmlEqualContribution}

\begin{abstract}
Contextual bandits are online learners that, given an input, select an arm and receive a reward for that arm. They use the reward as a learning signal and aim to maximize the total reward over the inputs. Contextual bandits are commonly used to solve recommendation or ranking problems. This paper considers a learning setting in which multiple parties aim to train a contextual bandit \emph{together} in a private way: the parties aim to maximize the total reward but do not want to share any of the relevant information they possess with the other parties. Specifically, multiple parties have access to (different) features that may benefit the learner but that cannot be shared with other parties. One of the parties pulls the arm but other parties may not learn which arm was pulled. One party receives the reward but the other parties may not learn the reward value. This paper develops a privacy-preserving multi-party contextual bandit for this learning setting by combining secure multi-party computation with a differentially private mechanism based on epsilon-greedy exploration.
\end{abstract}

\section{Introduction}
\label{sec:introduction}
Contextual bandits are an important learning paradigm used by many recommendation systems~\citep{langford2008bandits}. The paradigm considers a series of interactions between the learner and the environment: in each interaction, the learner receives a \emph{context feature} and selects an \emph{arm} based on that context. The environment provides the learner with a \emph{reward} after the arm is pulled (\emph{i.e.}, an action is executed). In traditional contextual bandit scenarios, the learner is a single \emph{party}: that is, the party that pulls the arm is also the party that has access to all context features and that receives the reward. In many practical scenarios, however, contextual bandit learning involves multiple parties: for example, recommendation systems may involve content producers, content consumers, and the party that operates the recommendation service itself. These parties may not be willing or allowed to share all the information with each other that is needed to produce high-quality recommendations. For instance, a travel-recommendation service could recommend better itineraries by taking prior airline bookings, hotel reservations, and airline and hotel reviews into account as context. To do so, the travel-recommendation service requires data from booking, reservation, and review systems that may be operated by other parties. Similarly, a restaurant-recommendation service may be improved by considering a user's prior reservations made via a restaurant-reservation service operated by another party. 

In this paper, we develop a privacy-preserving contextual bandit that learns models in such multi-party settings. We study a multi-party contextual bandit setting in which: (1) all parties may provide some of the context features but none of the parties may learn each other's features, (2) the party that pulls the arm is the only one that may know which arm was pulled, and (3) the party that receives the reward is the only one that may observe the reward value. We develop a learning algorithm that combines techniques from secure multi-party computation~\citep{ben1988completeness} and differential privacy~\citep{dwork2006}. The algorithm achieves a high degree of privacy with limited losses in prediction accuracy compared to a non-private learner by using exploration mechanisms that naturally provide differential privacy. We provide theoretical guarantees on the privacy of our algorithm and empirically demonstrate its efficacy.

\section{Problem Statement}
\label{sec:problem_statement}
\noindent \textbf{Learning setting.}
We consider a multi-party contextual bandit setting (\emph{cf.} Figure~\ref{fig:ppcb}) with a set of parties $\mathcal{P}$,
a finite and fixed set of arms $\mathcal{A}$, and $T$ iterations. Each
iteration, $t$, in our learning setting consists of five main stages:
\begin{enumerate}[leftmargin=*]
\item Each party $p \in \mathcal{P}$ provides context features $\bx_{t,p} \in \mathbb{R}^{D_p}$ to the learner in a privacy-preserving way. We assume features are samples from a context distribution, $p(\bx)$.
\item The parties select an arm $a_t$ by jointly evaluating policy $\pi(\bx_{t,1}, \dots, \bx_{t,|\mathcal{P}|})$. Action $a_t$ is not revealed to any of the parties, and $\bx_{t,q}$ is not revealed to parties $p \neq q$.
\item The parties in $\mathcal{P}$ reveal the selected action, $a_t$, to party $p' \notin \mathcal{P}$ but not to any of the other parties. Party $p'$ pulls the corresponding arm.
\item Party $p'' \notin \mathcal{P}$ receives reward $r_t \in \mathbb{R}$, which is a sample from reward distribution $p(r | a_t, \bx_t)$, from the environment. Party $p''$ does not reveal $r_t$ to any other party.
\item The parties in $\mathcal{P}$ update policy $\pi$ without having access to the reward, $r_t$, or each other's contexts, $\bx_{t,p}$.
\end{enumerate}

\begin{figure}[t]
\centering
\includegraphics[width=\linewidth]{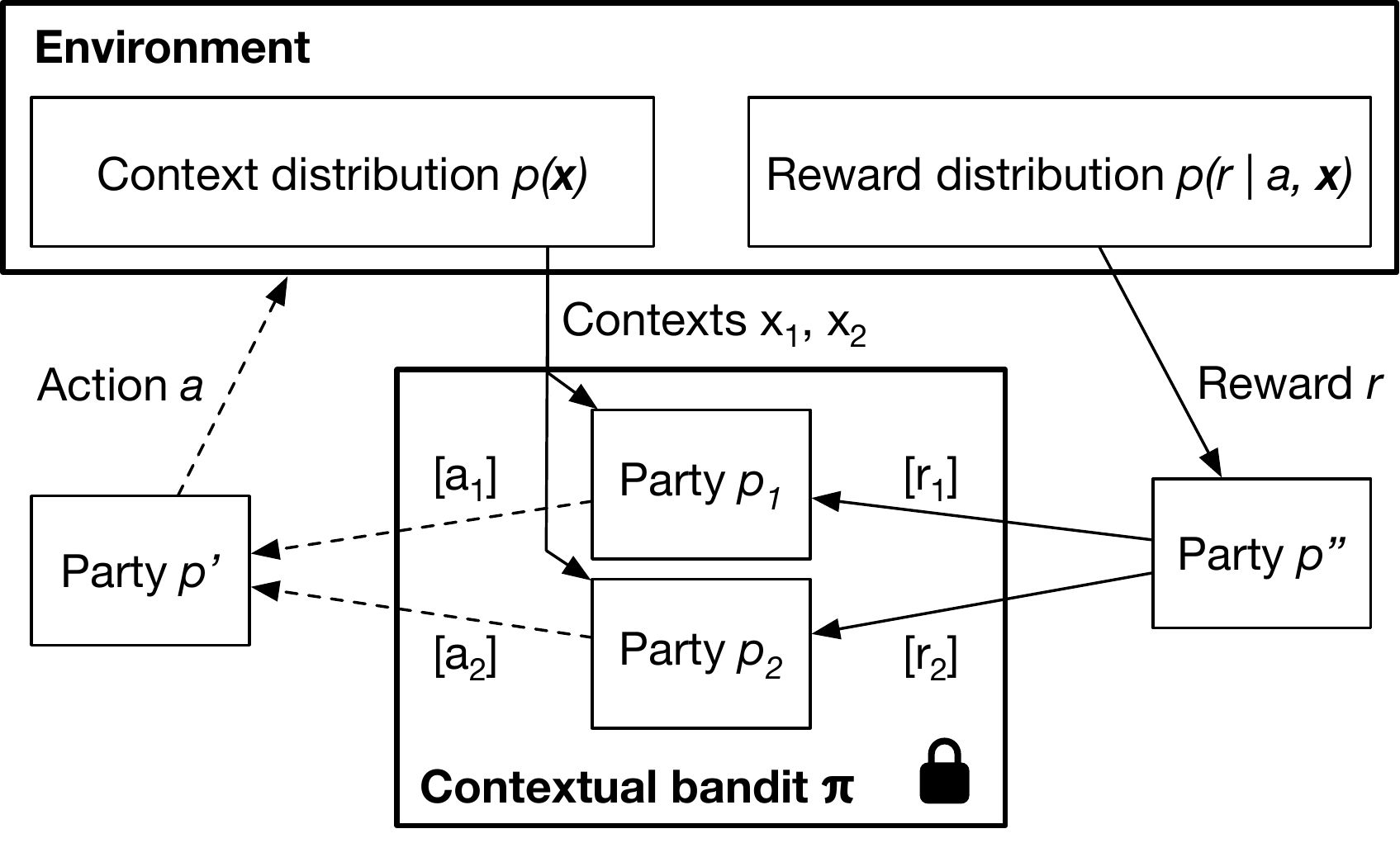}
\caption{Illustration of our privacy-preserving, multi-party contextual bandit learning
    setting. Parties $p_1 \in \mathcal{P}$ and $p_2 \in \mathcal{P}$ collaborate to learn a contextual-bandit policy, $\pi$, without leaking information to each other. The lock icon indicates parts of the learner in which computations are performed on encrypted information. Dashed lines
    represent differentially private information. Variables
    in square brackets denote secret shares of data.}
\label{fig:ppcb}
\end{figure}

Together, the parties learn a policy $\pi(\bx_{t,1}, \dots, \bx_{t,|\mathcal{P}|}) \in \Pi$ that maximizes the average \emph{reward} over all $T$ iterations:
\begin{equation}
V(\pi) = \frac{1}{T} \sum_{t=1}^T \mathbb{E}\left[r_t | \pi(\bx_{t,1}, \dots, \bx_{t,|\mathcal{P}|}), \bx_{t,1}, \dots, \bx_{t,|\mathcal{P}|}\right],\nonumber
\end{equation}
where the expectation is over the context distribution, $p(\bx)$, and reward distribution, $p(r|
a, \bx)$, that are given by the environment. We assume these distributions are fixed. We also assume the reward, $r$, has a fixed range, $r \in [0, 1]$.

Our policy set, $\Pi$, is the set of all epsilon-greedy policies that use models with a linear relation between
context features $\bx_{t,p}$ and the corresponding score for arm
$a_t$~\citep{li2010contextual}. In particular, let $\bx_t \in \mathbb{R}^D$ be
the concatenated context features of all $|\mathcal{P}|$ parties. To compute the score for
arm $a$ at iteration $t$ we use a linear model $s_a = \bw_a^\top \bx_t$. Let $\bX_a$
be the $N_t \!\times\! D$ design matrix at iteration $t$ which consists of the $N_t$
context vectors observed before pulling arm $a$. Similarly, $\mathbf{r}_a \in
\mathbb{R}^{N_t}$ is the vector of observed rewards after pulling that same arm. The learner aims to find the model parameters
at iteration $t$ by minimizing the least-squares error $\|\bX_a \bw_a -
\mathbf{r}_a\|_2^2$ using the linear least-squares solution:
$\bw_a = \bW_a^{-1}\bb_a$, where $\bW_a = \bX_a^\top \bX_a$ and $\bb_a = \bX_a^\top
\mathbf{r}_a$.

\noindent \textbf{Security model.}
In line with the cooperative nature of our learning setting, we assume a
\emph{honest-but-curious} security model~\citep{goldreich2009}: we assume parties do not collude and
follow the specified learning algorithm, but parties may try and learn as much as
possible from the information they observe when executing the algorithm.

Our learning algorithm comes with a differential-privacy guarantee on the information
that parties $p'$ and $p''$ can obtain on context features $\forall p \in
\mathcal{P}: \bx_{t,p}$. The action must be revealed to party $p'$ so that it
can pull the corresponding arm. Similarly, party $p''$ must receive the reward.
Hence, some information is ultimately revealed to parties
$p'$ and $p''$. We provide a differential-privacy guarantee on the information leaked about $\bx_t$ by exploiting
the randomness introduced by the epsilon-greedy exploration of the bandit algorithm.

Our algorithm guarantees that the other parties, $p \in \mathcal{P}$, do not gain any additional knowledge
about context features $\bx_{t,q}$ with $q \!\neq\!  p$.
To obtain this guarantee, the algorithm does assume that all parties have access to privately shared random
numbers generated during an \emph{off-line} phase.\footnote{These can be
generated by a trusted third party, secure enclave, or additive homomorphic encryption \citep{brakerski2012leveled}.}

\section{Privacy-Preserving Contextual~Bandits}
\label{sec:private_bandits}
Our privacy-preserving, multi-party contextual bandit employs an
epsilon-greedy policy that assumes a linear relation between the context
features and the score for an arm. To obtain privacy guarantees, we use
\emph{arithmetic secret sharing} techniques commonly used in secure multi-party
computation to implement our learner \citep{spdz}. We rely on the
differentially private properties of epsilon-greedy policies when performing
actions (see Section~\ref{sec:guarantees}).

In arithmetic secret sharing, a scalar value $x \in \thefield$ (where $\thefield$ denotes a ring with $Q$ elements, and $Q$ is large) is shared across $|\mathcal{P}|$ parties such that the sum of the
shares reconstructs the original value $x$. We denote the secret sharing
of $x$ by $[x] = \{\forall p \in \mathcal{P}: [x]_p\}$, where $[x]_p \in
\thefield$ indicates party $p$'s share of $x$. The representation has the
property that $\sum_{p \in \mathcal{P}} [x]_p \mod Q \!=\! x$. We use a simple encoding
to represent the real-valued context features and model weights in
$\thefield$: to obtain $x \in \thefield$, we multiply $x_R \in \mathbb{R}$ with a large scaling factor $B$ and round to the nearest integer: $x = \lfloor B x_R \rceil$, where $B = 2^L$ for some
precision parameter, $L$. We decode an encoded value, $x$, by
computing $x_R \approx x / B$. Encoding real values this way incurs a precision loss
that is inversely proportional to $L$.

Party $p \in \mathcal{P}$ shares a context feature, $x$, by drawing $|\mathcal{P}|-1$ numbers uniformly at
random from $\thefield$ and distributing them among the other parties. The random numbers
are the shares, $\forall q \neq p: [x]_q$, for those parties. Subsequently,
party $p$ computes its own share as $[x]_p = x - \sum_{q \neq p} [x]_q$. None of the other parties can infer any information about $x$ from
their share. The policy is evaluated using exclusively arithmetically shared data and parameters. To allow party $p'$ to pull an arm $a_t \in \mathcal{A}$, all parties
communicate their shares, $[a_t]_p$, to party $p'$, which computes $a_t = \sum_{p \in |\mathcal{P}|}
[a_t]_p$ and performs the corresponding action, $a_t$.
Subsequently, party $p''$ receives the reward, $r_t$, which it secret shares with parties $p \in \mathcal{P}$. The parties use $[r_t]$ to update the policy, again, by performing computations exclusively on arithmetic shares.

Algorithm~\ref{alg:bandits} gives an overview of our privacy-preserving contextual bandit
learner, which follows \citet{li2010contextual} but uses epsilon-greedy rather than UCB exploration. Unless otherwise stated, all computations in the algorithm are
performed by all parties in $\mathcal{P}$ \emph{on the secret-shared
values} without leaking any information to the other parties. 

\begin{algorithm*}[ht!]
  \caption{Privacy-preserving contextual bandit. \footnotesize{(Follows \citet{li2010contextual} but uses epsilon-greedy rather than UCB exploration.)}}
  \label{alg:bandits}
\begin{algorithmic}
\STATE {\bfseries Input:} \\
- Exploration parameter, $\epsilon$.\\
- Party $p' \notin \mathcal{P}$ that pulls arms; party $p'' \notin \mathcal{P}$ that receives rewards.\\
- Set of parties, $\mathcal{P}$; set of arms, $\mathcal{A}$; number of context features produced by all parties, $D$.\\
- A store of Beaver triples and private (Bernoulli and uniform) samples generated off-line.
\STATE {\bfseries Output:} Secretly shared weights, $[\W]$, and biases, $[\B]$.

Party $p'$ initializes weights $\W = \{\bW_1, \dots, \bW_{| \mathcal{A} |}\}$ with $\bW_a \leftarrow \mathbf{I}_{D \times D}$.\\
Party $p'$ initializes biases $\B = \{\bb_1, \dots, \bb_{| \mathcal{A} |}\}$ with $\bb_a \leftarrow \mathbf{0}_{D \times 1}$.\\
Party $p'$ secretly shares weights $[\W] = \{[\bW_1], \dots, [\bW_{| \mathcal{A} |}]\}$ and biases $[\B] = \{[\bb_1], \dots, [\bb_{| \mathcal{A} |}]\}$.\\
\FOR{$t \in \{1, \dots, T\}$}
\STATE
    Parties observe context features $\forall p \in \mathcal{P}: \bx_{t,p} \in \mathbb{R}^{D_p}$ (with $\sum_{p \in \mathcal{P}} D_p = D$).\\
	Parties secretly share context features $\forall p \in \mathcal{P}: [\bx_{t,p}]$.\\
	Parties concatenate shares $\forall p \in \mathcal{P}: [\bx_{t,p}]$ into single share $[\bx_t]$ with $\bx_t \in \mathbb{R}^D$.\\

    Select next private sample $[y_t]$ with $y_t \sim Bernoulli(\epsilon)$.

    \FOR{$a \in \mathcal{A}$}
    \STATE
		Compute $[\bw_a] \leftarrow [\bW_a^{-1} \bb_a]$ (weights for arm $a$).\\
		Compute $[s_{t,a}] \leftarrow [\bw_a^\top \bx_{t}]$ (score for arm $a$).\\

        Select next random private sample $[v_a]$ with $v_a \sim Uniform(0, 1)$.\\
        Compute $[\tilde{s}_{t,a}] \leftarrow [y_t v_a + (1- y_t) s_{t,a}]$ (differentially private score for arm $a$).\\
    \ENDFOR

	Compute $[a_t] \leftarrow [\argmax_a \tilde{s}_{t,a}]$ (arm to be pulled).\\
	Party $p'$ opens up $a_t$ and pulls arm.\\
	Party $p'$ constructs $\forall a \in \mathcal{A}: o_a \leftarrow \mathbb{I}(a = a_t)$ (binary values indicating selected action).\\
	Party $p'$ secretly shares all indicator variables, $\forall a \in \mathcal{A}: [o_a]$.\\
	Party $p''$ receives reward $r_t$ and secretly shares $[r_t]$.\\

	\FOR{$a \in \mathcal{A}$}
    \STATE
        Compute $[\bW_a] \leftarrow [\bW_a + o_a \bx_t \bx_t^\top ]$.\\
    	 Compute $[\bb_a] \leftarrow [\bb_a + o_a r_t \bx_t]$.\\
    \ENDFOR
\ENDFOR
\end{algorithmic}
\end{algorithm*}

The algorithm relies on the homomorphic properties of arithmetic secret sharing in order to
perform computations directly on the encrypted data. Below, we give an overview
of how these computations are implemented. The primary cost in
executing Algorithm~\ref{alg:bandits} is the number of \emph{communication rounds} 
between parties for certain operations, such as the evaluation of the $\argmax$ function. This communication can sometimes be overlapped
with other computations.

\noindent\textbf{Addition.} The addition of two encrypted values, $[z] = [x] + [y]$, can be trivially implemented by having each party $p$ sum their shares of $[x]$ and $[y]$. That is, each party $p \in \mathcal{P}$ computes $[z]_p \leftarrow [x]_p + [y]_p$.

\noindent\textbf{Multiplication.} To facilitate multiplication of two secret
shared values, the parties use random Beaver triples \citep{beaver1991} that were generated in an off-line preprocessing phase.
A Beaver triple of secret shared values $([a], [b], [c])$ satisfies the property $c \!=\!
ab$. The parties use the Beaver triple to compute $[\alpha] = [x] - [a]$ and
$[\beta] = [y] - [b]$ and decrypt $\alpha$ and $\beta$. This does not leak
information if $a$ and $b$ were drawn uniformly at random from the ring $\thefield$.
The product $[x][y]$ can now be evaluated by computing $[c] + \alpha[b] +
\beta[a] + \alpha\beta$. It is straightforward to confirm that the result of the
private multiplication is correct:
\begin{multline}
\label{eq:beaver}
[c] + \alpha[b] + \beta[a] + \alpha\beta= \\
    [a][b] + [x][b] - [a][b] + [y][a]-[b][a]+([x]-[a])([y]-[b])\\
    = [x][y].\nonumber
\end{multline}
To decrypt $\alpha$ and $\beta$, all parties to communicate their shares of $[\alpha]$ and $[\beta]$ to each other: a communication round. The required correction for the additional scaling term from the encoding, $B$, incurs a second communication round.

\noindent\textbf{Square.} To compute the square $[x^2]$, the parties use a Beaver pair $([a], [b])$ such
that $b \!=\! a^2$. Akin to before, the parties use the Beaver pair compute $[\alpha] = [x] - [a]$, decrypt $\alpha$, and obtain the result via $[x^2] = [b]
+ 2\alpha[a] + \alpha^2$.

\noindent\textbf{Dot product, matrix-vector, and matrix-matrix multiplication.} The operations on scalars we described above can readily be used to perform operations on vectors and matrices that are secret-shared in element-wise fashion. Specifically, dot products combine multiple element-wise multiplications and additions. Matrix-vector and matrix-matrix multiplication is implemented by repeated computation of dot products of two arithmetically secret-shared vectors.

\noindent\textbf{Matrix inverse.} At each round, the algorithm computes $|\mathcal{A}|$ matrix inverses of the $D\!\times\!D$ matrices $\bW_a$, which is computationally costly. Because we only perform rank-one updates of each $\bW_a$ matrix, we can maintain a representation of $\bW_a^{-1}$ instead, rendering the matrix inversion in the algorithm superfluous, and use the Sherman-Morrison formula \citep{bartlett1951} to perform the parameter update:
\begin{equation}
    \label{eq:sherman}
\left[\bW_a^{-1}\right] \leftarrow \left[\bW_a^{-1} - o_a \frac{\bW_a^{-1} \bx_t \bx_t^\top \bW_a^{-1}}{1 + \bx_t^\top \bW_a^{-1} \bx_t}\right].
\end{equation}
This expression comprises only multiplications, additions, and a reciprocal (see below).

\noindent\textbf{Reciprocal.} We compute the reciprocal $[1/x]$ using a
Newton-Rhapson approximation with iterates $[x_{t+1}] \leftarrow \left[2 x_t -
x x_t^2 \right]$. The Newton-Rhapson approximation converges rapidly
when $x_0$ is initialized well. The reciprocal is only used in the Sherman-Morrison formula, so we choose the initialization with
this in mind. Because $\bW_a^{-1}$ in Equation~\ref{eq:sherman} is the inverse of a positive-definite matrix, it
is itself positive definite; the denominator in Equation~\ref{eq:sherman}, therefore, lies in the range $[1, c]$. Empirically, we found that $c < 10$ and that an initial value of $x_0 = 3e^{-(x - 0.5)} + 0.003$ leads to good approximations in this range.


\begin{algorithm}[t]
\caption{Privacy-preserving identification of all maximum elements in a secret-shared vector.}
\label{alg:argmaxes}
\begin{algorithmic}
\STATE {\bfseries Input:} An arithmetically secret-shared vector $[\bx]$.
\STATE {\bfseries Output:} A secret shared vector $[\by] \in \{0, 1\}^{| \mathcal{A} |}$ with ones indicating the maximum values in $[\bx]$.

\FOR{$i, j \in \mathcal{A} \times \mathcal{A}$}
\STATE
  Compute $[\ell_{i,j}] \leftarrow [\bx_i \geq \bx_j]$. \\
\ENDFOR \\
Set $[\by] \leftarrow [\mathbf{1}]$. \\
\FOR {$i \in \mathcal{A}$}
    \FOR{$j \in \mathcal{A}$}
        \STATE
        Compute $[y_i] \leftarrow [y_i][\ell_{i, j}]$. \\
    \ENDFOR
\ENDFOR
\end{algorithmic}
\end{algorithm}

\noindent\textbf{Argmax.} We compute the index of the maximum value of a vector
$[\bx]$ where $\bx \in (\thefield)^{|\mathcal{A}|}$ as a one-hot vector of the same
size as $\bx$. Our algorithm for evaluating the $\argmax$ of $[\bx]$, called $[\bz]$, has three main stages:
\begin{enumerate}[leftmargin=*]

    \item Use Algorithm~\ref{alg:argmaxes} to construct a vector $[\by]$ of the same length as $[\bx]$ that
        contains ones at the indices of all maximum values of $[\bx]$ and zeros elsewhere.

    \item Break ties in $[\by]$ by multiplying it element-wise with a random permutation of $\{1, \dots, |\mathcal{A}|\}$. This
        permutation, $[\boldsymbol{\gamma}]$, is generated and securely shared off-line.

    \item Use Algorithm~\ref{alg:argmaxes} to construct a one-hot vector $[\bz]$ that indicates the maximum value of
        $[\by \circ \boldsymbol{\gamma}]$.
\end{enumerate}
\vspace{-2mm}

The permutation in step 2 randomly breaks ties for the index
of the maximum value. We opt for random tie-breaking because breaking ties deterministically may leak information. For
example, if ties were broken by always selecting the last maximum value from $[\bx]$, the adversary would learn that $[\bx]$ did not have multiple maximum values if it observed a $[\bz]$ that has a $1$ as its first element.

In Algorithm~\ref{alg:argmaxes}, the evaluation of all $[\bx_i \geq \bx_j]$ terms
is performed on a binary secret share of $\bx_i$ and $\bx_j$.
A binary secret share is a is a special type of arithmetic secret sharing for binary data in which the ring size $Q = 2$ \citep{gmw}. To convert an arithmetic
share $[\bx]$ into a binary share $\langle \bx \rangle$, each party first secretly
shares its arithmetic share with the other parties and then performs addition
of the resulting shares. To construct the binary share $\langle [\bx]_p \rangle$
of its arithmetic share $[\bx]_p$, party $p \in \mathcal{P}$: (1) draws $|\mathcal P|-1$ random bit strings
$\langle [\bx]_p \rangle_q$ and shares those with the other parties and (2)
computes its own binary share $\langle [\bx]_p \rangle_p = \bigoplus_{q \neq p}
\langle [\bx]_p \rangle_q$. The parties now each obtained a binary share of $[\bx]_p$
without having to decrypt $\bx$. This process is repeated for each party $p \in \mathcal{P}$ to create
binary shares of all $|\mathcal{P}|$ arithmetic shares $[\bx]_p$. Subsequently, the
parties compute $\langle \bx \rangle = \sum_{p\in \mathcal P} \langle[\bx]_p\rangle$. The summation is implemented by a Ripple-carry adder in $\log_2 (|\mathcal P| \log_2 Q)$ rounds \citep{catrina2010improved}.

Subsequently, the $[\bx_i \geq \bx_j]$ operation is performed by computing $[\by] \leftarrow [\bx_i] - [\bx_j]$, constructing the binary secret sharing $\langle \by \rangle$ per the procedure outlined above, obtaining the most
significant bits, $\langle \by \rangle^{(MSB)}$, and converting those bits back to an arithmetic share. To convert from a binary share
$\langle \by \rangle$ to an arithmetic share $[\by]$, the parties compute $[\by] =
\sum_{b=1}^{B'} 2^b \left[\langle \by\rangle^{(b)}\right]$, where $\langle \by\rangle^{(b)}$
contains the $b$-th bits of the binary share $\langle \by\rangle$ and $B'$ is the total number of bits in the shared secret. To create the arithmetic
share of a bit, $\left[\langle \by \rangle^{(b)}\right]$, each party $p \in \mathcal{P}$
draws bits uniformly at random, and shares the difference between
their bits and the random bits with the other parties. The parties sum all resulting shares to
obtain $\left[\langle \by \rangle^{(b)}\right] = [\bx_i \geq \bx_j]$.

Overall, the evaluation of $[\bx_i \geq \bx_j]$ requires seven communication rounds. We parallelize the reduction over $i$, and perform the reduction over $j$ using a binary reduction tree in $\lceil \log_2 \left(|\mathcal{A}|-1\right) \rceil$ communication rounds.

\begin{figure}[t]
\centering
\includegraphics[width=\linewidth]{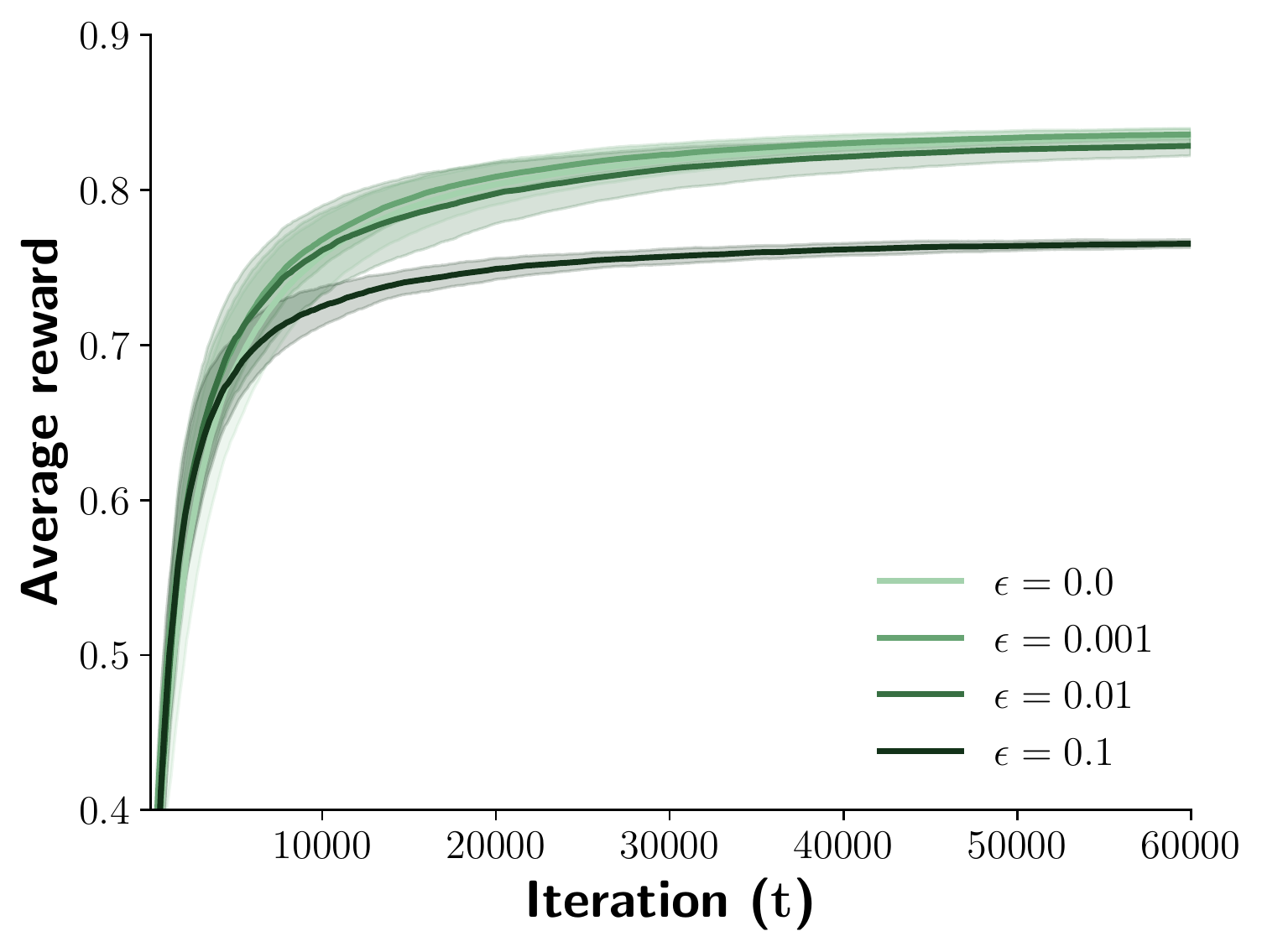}
\caption{Average reward (over 10 runs) of privacy-preserving contextual bandits
during training for four different values of $\epsilon$. Shaded regions indicate the standard deviation across runs.}
\label{fig:bandits_epsilon}
\end{figure}

\begin{figure}[t]
\centering
\includegraphics[width=1.005\linewidth]{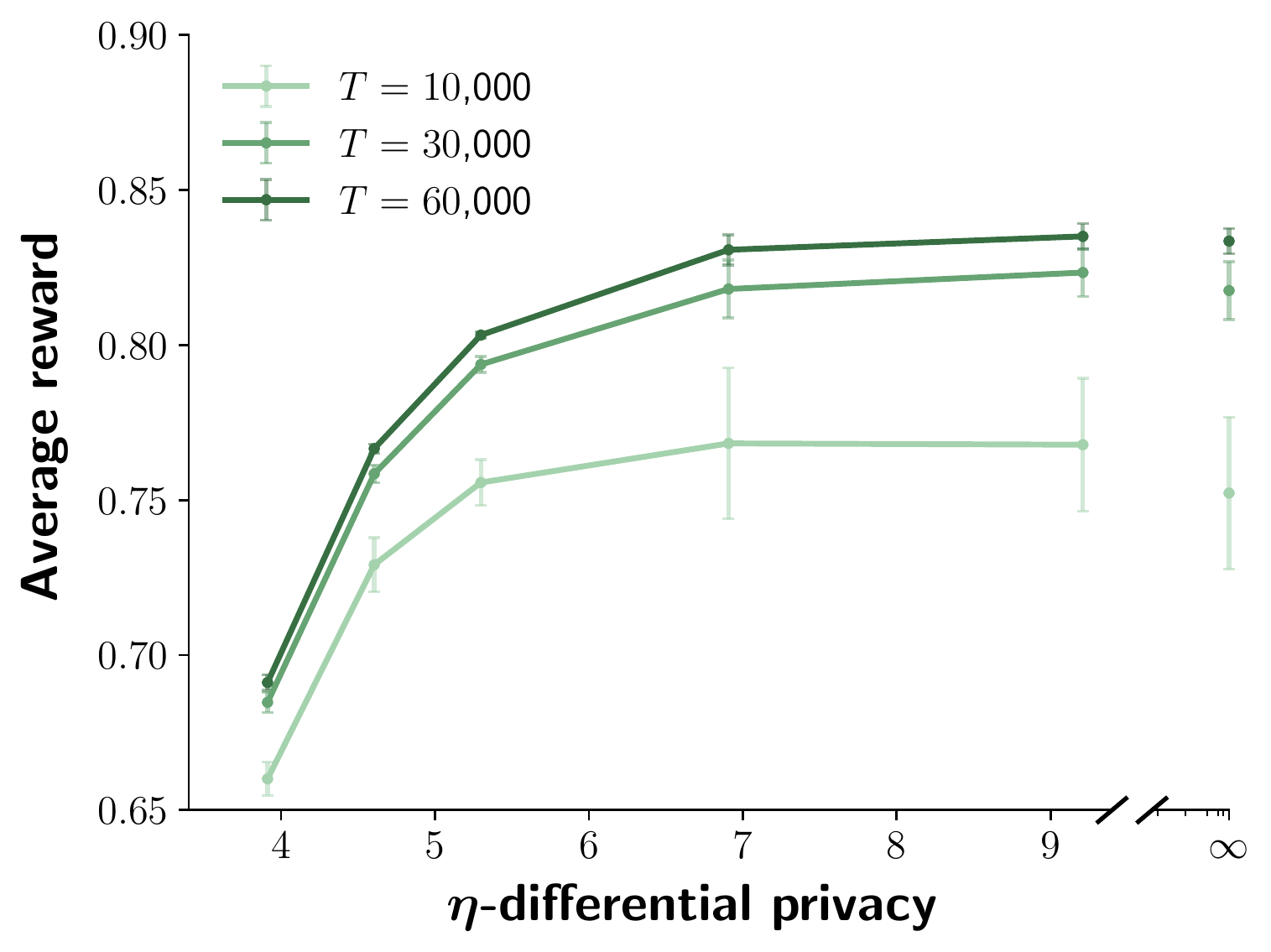}
\caption{Average reward as a function of differential privacy, $\eta = \log(|\mathcal{A}| / \epsilon)$. The differential privacy of the algorithm decreases as $\eta$ increases; the value $\eta = \infty$ corresponds to the non-private version of the algorithm.}
\label{fig:reward_privacy}
\end{figure}

\section{Privacy Guarantee}
\label{sec:guarantees}
The privacy guarantees for our algorithm rely on: (1) well-known guarantees on
the security of arithmetic and binary secret sharing mechanisms and (2) the
differentially private opening of actions $a_t \in \mathcal{A}$ by party $p'$. For security
guarantees of secret sharing, we refer the reader to \citet{spdz}. We focus on the
differentially private opening of actions. Our primary observation is a natural
link between epsilon-greedy policies and differential privacy. 

A mechanism $\mathcal{M}$ is $\eta$-differentially
private if for all datasets $\D$ and $\D'$ that
differ by a single example and for all output sets $\mathcal{S} \subseteq \text{Range}(\mathcal{M})$ the following holds \citep{dwork2011differential}:
\begin{equation*}
P(\mathcal{M}(\D) \in \mathcal{S}) \le e^\eta P(\mathcal{M}(\D') \in \mathcal{S}).
\end{equation*}

\begin{theorem}
\label{thm:eps_greedy_dp}
If only the selected action $a_t \in \mathcal{A}$ at round $t$ is revealed, then
a policy that uses $\epsilon$-greedy exploration is $\log(|\mathcal{A}| / \epsilon)$-differentially
private.
\end{theorem}

\begin{proof}
The probability of selecting the action that corresponds to the maximum score is given by $P\left(\pi(\bs) = \argmax_{a \in \mathcal{A}} s_a\right) = (1 - \epsilon) + \frac{\epsilon}{|\mathcal{A}|}$. We use this probability to bound the privacy loss, $\eta$:
\begin{align*}
    \log &\frac{P(\pi(\bs) = j \mid j = \argmax_{a \in \mathcal{A}} s_a)}{P(\pi(\bs) = k \mid k \ne \argmax_{a \in \mathcal{A}} s_a)} \\
    &= \log \frac{1 - \epsilon + \epsilon / |\mathcal{A}|}{\epsilon / |\mathcal{A}|} \\
    &= \log \left( \frac{|\mathcal{A}|}{\epsilon} - |\mathcal{A}| + 1 \right) \le \log \left(\frac{|\mathcal{A}|}{\epsilon} \right).
\end{align*}
Using the fact that the exploration parameter $\epsilon \in [0, 1]$, we also observe that:
\begin{align*}
    \log &\frac{P(\pi(\bs) = j \mid j \ne \argmax_{a \in \mathcal{A}} s_a)}{P(\pi(\bs) = k \mid k = \argmax_{a \in \mathcal{A}} s_a)} \\
    &= \log \frac{\epsilon / |\mathcal{A}|}{1 - \epsilon + \epsilon / |\mathcal{A}|} \\
    &= \log \left( \frac{\epsilon}{(1 - \epsilon)|\mathcal{A}| + \epsilon} \right) \\
    &\le \log 1 = 0.
\end{align*}
To complete the proof, we observe that:
\begin{equation*}
    \log \frac{P(\pi(\bs) = j \mid j \ne \argmax_{a \in \mathcal{A}} s_a)}{P(\pi(\bs) = k \mid k \ne \argmax_{a \in \mathcal{A}} s_a)} = \log \frac{\epsilon/|\mathcal{A}|}{\epsilon/|\mathcal{A}|} = 0.\qedhere
\end{equation*}
\end{proof}
The above result is a generalization
of the randomized response protocol \citep{warner1965} to $|\mathcal{A}| > 2$ arms and arbitrary $\epsilon \in
[0, 1]$.
Because the privacy loss grows logarithmically with the number of actions, we obtain
high differential privacy for reasonable settings of exploration parameter $\epsilon$.

\begin{figure}[t]
\centering
\includegraphics[width=\linewidth]{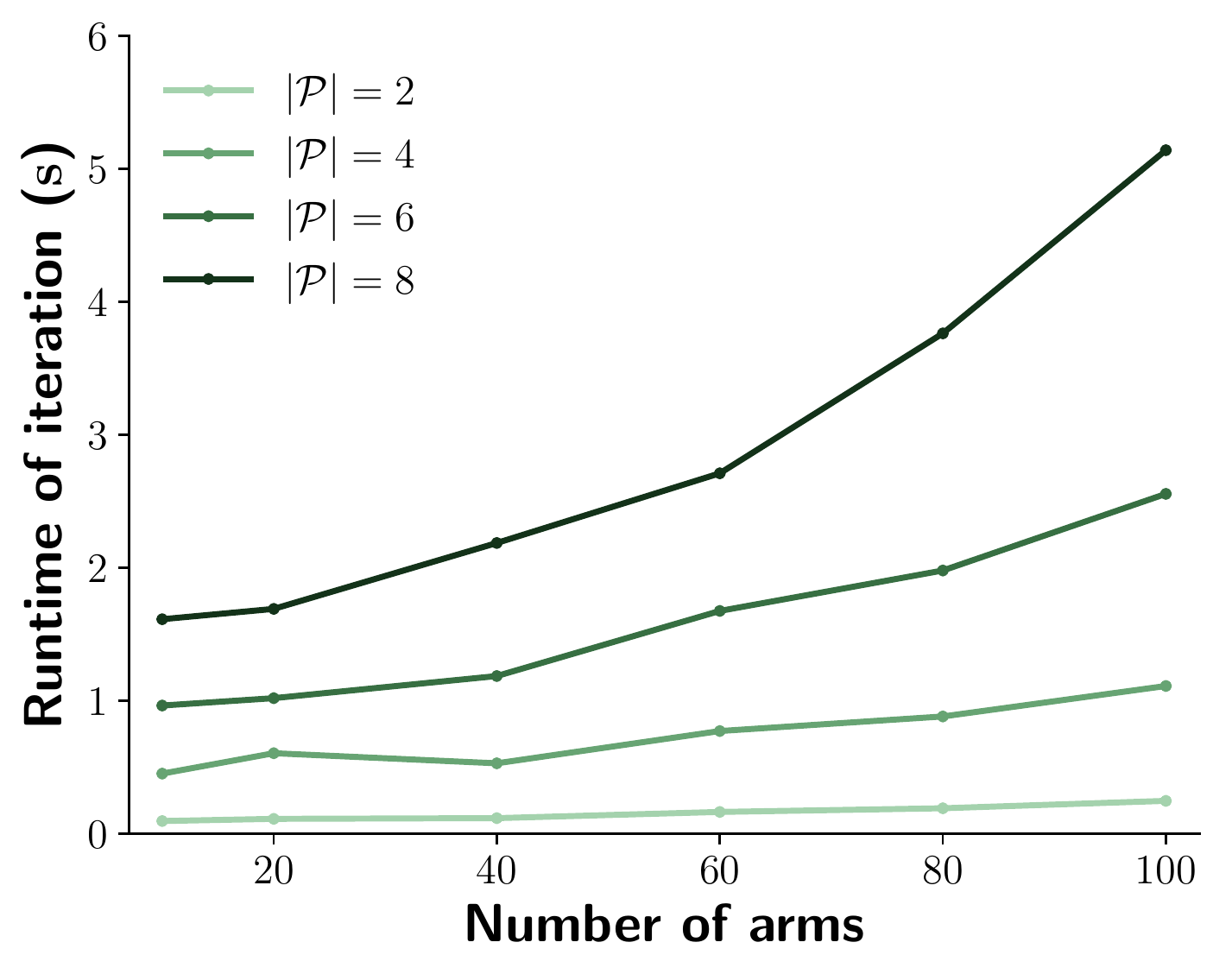}
\caption{Run-time of Algorithm~\ref{alg:bandits} as a function of the number of
arms, $|\mathcal{A}|$. Each line is for a different number of parties, $|\mathcal{P}|$.}
\label{fig:scaling_arms}
\end{figure}

\begin{figure}[t]
\centering
\includegraphics[width=\linewidth]{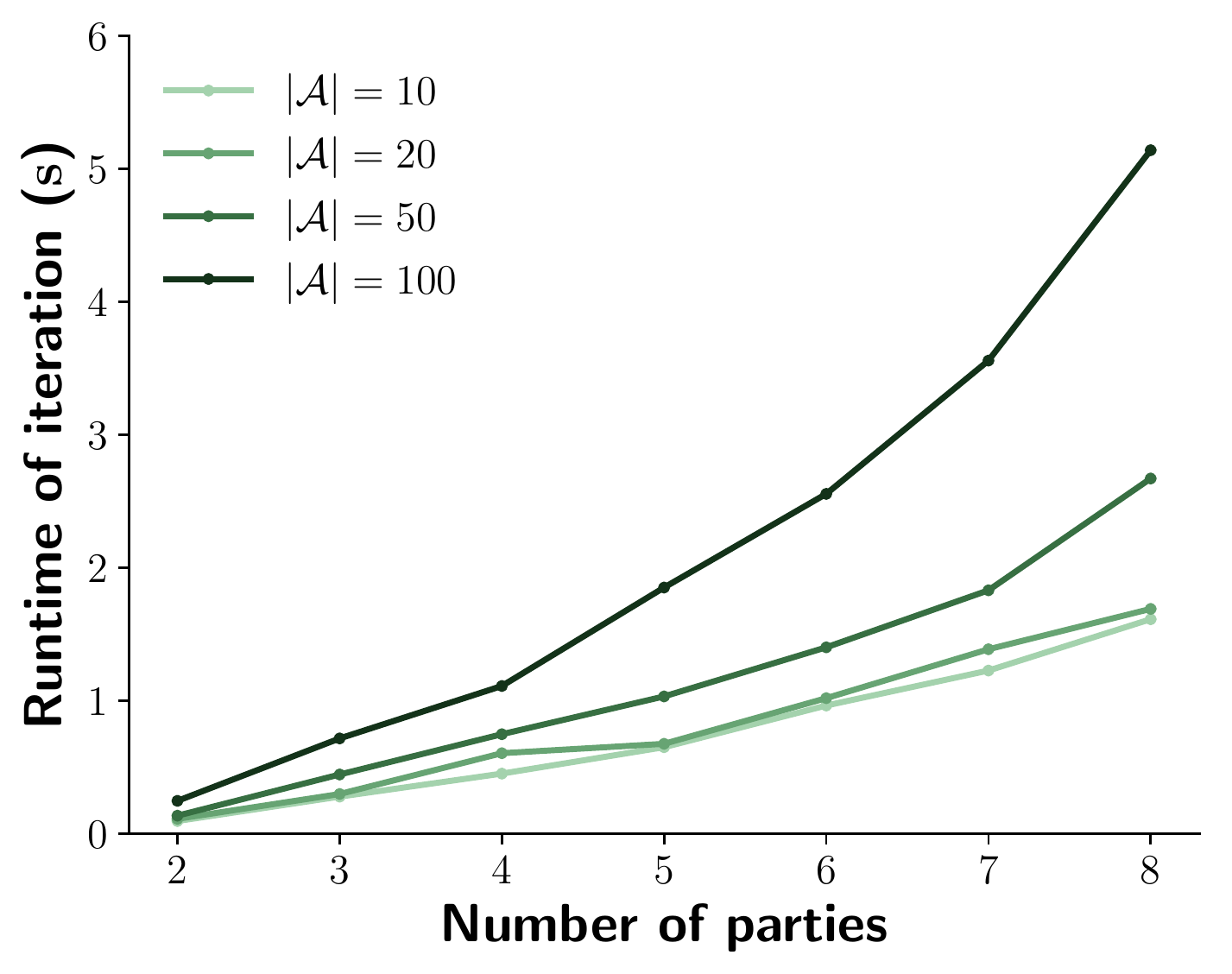}
\caption{Run-time of Algorithm~\ref{alg:bandits} as a function of the number of
parties, $|\mathcal{P}|$. Each line if for a different number of arms, $|\mathcal{A}|$.}
\label{fig:scaling_parties}
\end{figure}

\section{Experiments}
\label{sec:experiments}
We perform experiments on the MNIST dataset to evaluate the efficiency and reward-privacy trade-off of our algorithm. We reduce the data dimensionality by projecting each digit image to the first 20 principal components of the dataset, and normalize each resulting vector to have unit length. We perform a single sweep through the $60,000$ images in the MNIST training set. At each iteration, the parties receive a new (secret-shared) image and need to select one of $10$ arms (digit classes). The reward is $1$ if the selected arm
corresponds to the correct digit class, and $0$ otherwise. We implement Algorithm~\ref{alg:bandits} on a ring $\thefield$ with $Q \!=\! 2^{64}$. We rely on the property of 64-bit integer operations, where a result that is too large to fit in 64 bits is automatically mapped to $\thefield$ with $Q \!=\! 2^{64}$. We use $L \!=\! 20$ bits of precision to encode floating-point values into $\thefield$ and $7$ Newton-Rhapson iterations for
computing the reciprocal function. Most of our experiments are performed using $|\mathcal{P}| \!=\! 2$ parties. Code reproducing the results of our experiments is publicly available at \url{http://www.anonymized.com}.

\noindent\textbf{Reward and privacy.} Figure~\ref{fig:bandits_epsilon} shows
the average reward that our private contextual bandits obtain for four
different values of $\epsilon$. We compared the results obtained by our
algorithm to that of a non-private implementation of the epsilon-greedy
contextual bandit learner, and confirmed that the observed rewards are the same for a
range of $\epsilon$ values.

Figure~\ref{fig:reward_privacy} shows the average reward (averaged over $5$ runs) observed as a function of the differential privacy parameter, $\eta$ (higher values represent less privacy). The results were obtained by varying $\epsilon$ and are shown for experiments with three different dataset sizes, $T$. The results show that at certain levels of differential privacy, the reward obtained by the private algorithm is higher than that of its non-private counterpart ($\eta \!=\! \infty$). Indeed, some amount of exploration benefits the learner whilst also providing differential privacy. For higher levels of privacy, however, the reward obtained starts to decrease because too much exploration is needed to obtain the required level of privacy.

\begin{table}[b]
\vspace{-2.75mm}
\captionof{table}{Performance characteristics of key operations in Algorithm~\ref{alg:bandits} (for $|\mathcal{P}|\!=\! 2$ and $|\mathcal{A}| \!=\! 100$): number of communication rounds between parties (left) and slowdown compared to a non-private implementation of the learner (right).}
\label{table:ops}
\centering
\begin{tabular}{ l r r r }
\toprule
\bf Operation& \bf Rounds &  \bf Slowdown  \\
\midrule
\rowcolor{Gray} Addition        &  0  & $11\times$  \\
Multiplication                   & 2   & $380\times$  \\
\rowcolor{Gray} Reciprocal & 30   &  $6,000\times$  \\
    Argmax                     & $\mathcal{O}\left( |\mathcal{P}| + \log_{2} (|\mathcal{A}|)\right)$ &  $34,000\times$  \\
\bottomrule
\end{tabular}
\end{table}

\begin{figure}[t]
\centering
\includegraphics[width=\linewidth]{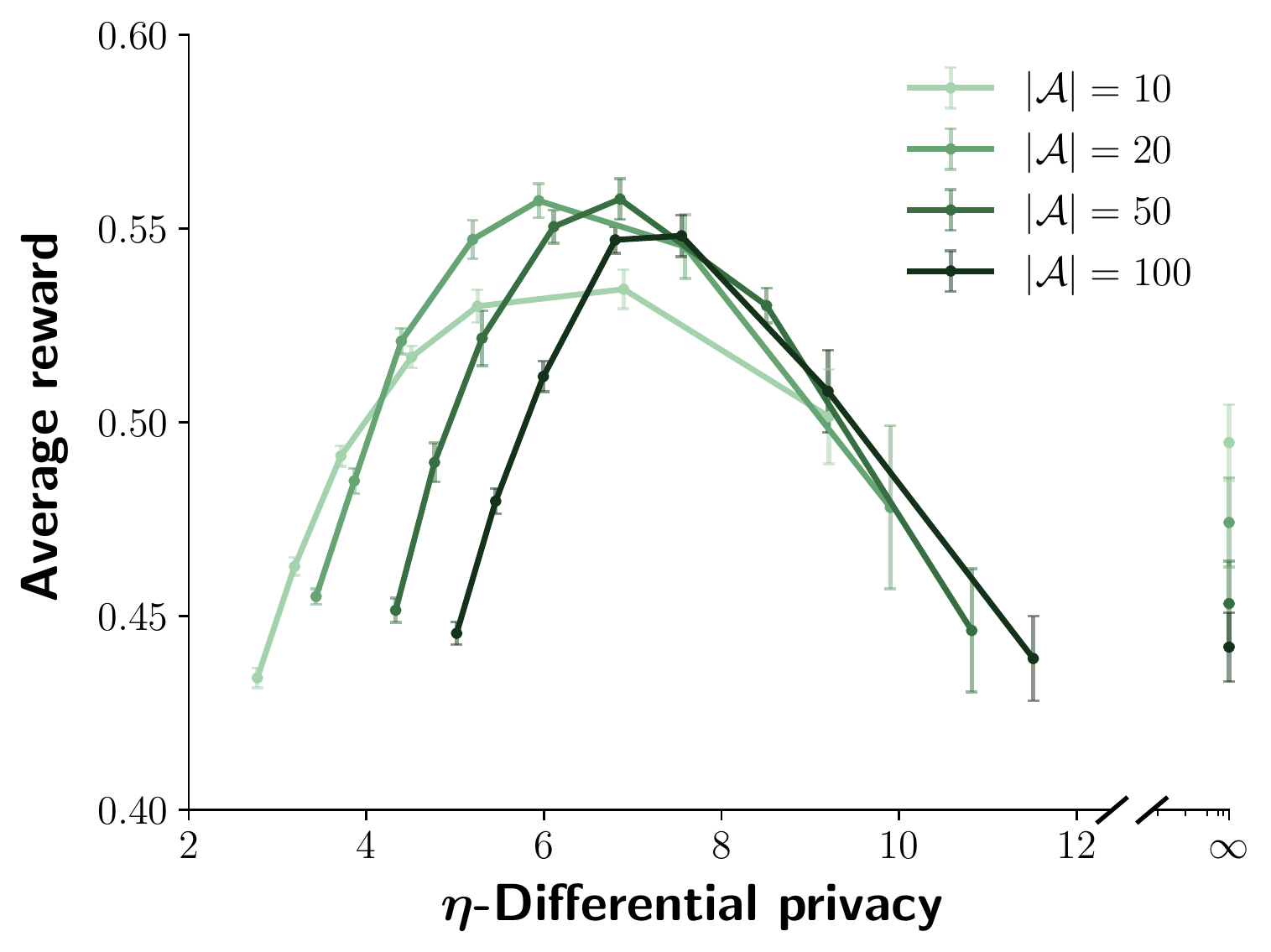}
\caption{Average reward as a function of differential privacy, $\eta =
\log(|\mathcal{A}| / \epsilon)$ for a variable number of arms, $|\mathcal{A}|$. Privacy decreases as $\eta$ increases; the value $\eta =
\infty$ corresponds to the non-private version of the algorithm.}
\label{fig:privacy_arms}
\end{figure}

\noindent\textbf{Efficiency and scale.}
Table~\ref{table:ops} reports the run-time of the key operations in Algorithm~\ref{alg:bandits}, compared to an implementation of the same operations in PyTorch. In our experiments, the private contextual-bandit implementation with $|\mathcal{P}| \!=\! 2$ parties and $|\mathcal{A}| \!=\! 10$ actions is nearly
500$\times$ slower than a regular implementation. In real-world settings, the slow-down would likely be even higher because of network latency: our experiments were performed on a single machine where each party is implemented as a separate process. There are two key sources of inefficiency in
Algorithm~\ref{alg:bandits}:
\vspace{-2mm}
\begin{enumerate}[leftmargin=*]
\item The weight update is
$\mathcal{O}(D)$ in a regular contextual-bandit implementation (only the
weights for the selected arm are updated) but $\mathcal{O}(|\mathcal{A}| D)$ in
Algorithm~\ref{alg:bandits}: the private implementation cannot reveal the
selected arm and, therefore, has to update all the weights. We note that
these weight updates are parallelizable over arms.
\item Some operations
(\emph{e.g.}, reciprocal and argmax) require additional computation and
communication.
\end{enumerate}

We also perform experiments in which we vary the number of arms, $|\mathcal{A}|$. To
increase the number of arms, we construct a $K$-means clustering of the dataset and set $|\mathcal{A}| \!=\!
K$.  We define the rewards to be Bernoulli-distributed with $P(r_a \!=\! 1)
\!=\! \nu_a$. The probabilities $\nu_a$ are Gaussian-kernel values based on the
distances from a data point to the inferred cluster centers: $\nu_a \propto e^{-\frac{1}{2\sigma} \|c_a - x\|_2}$,
where $c_a$ is the $a$-th cluster and $\sigma$ is used to control the
difficulty of the problem. We set $\sigma\!=\! 1/2$ in our experiments. Figure~\ref{fig:scaling_arms} demonstrates how the contextual-bandit algorithm scales with the number of arms.  For small numbers of arms ($|\mathcal{A}|
\le 20$), the implementation overhead dominates the computation time. For larger
numbers of arms ($|\mathcal{A}| \ge 40$), we observe quadratic scaling. Figure~\ref{fig:scaling_parties} shows how the
algorithm scales as a function of the number of parties, $|\mathcal{P}|$. The results illustrate that the run-time of our algorithm is $\mathcal{O}(|\mathcal{P}|^2)$: all parties communicate with each other in every communication round, which leads to the quadratic scaling observed.

Figure~\ref{fig:privacy_arms} demonstrates how the reward changes as a function
of the privacy loss $\eta$ when the number of arms, $|\mathcal{A}|$, is varied. The privacy loss increases logarithmically in the number of
arms, but the amount of
exploration needed also increases. As a result, the optimal privacy loss in
terms of reward only tends to increase slightly as the number of arms in the bandit increases. Indeed, this increase may be prohibitively large for web-scale recommendation applications in which the bandit has to select one arm out of millions of arms.

\begin{figure}
\centering
\includegraphics[width=1.04\linewidth]{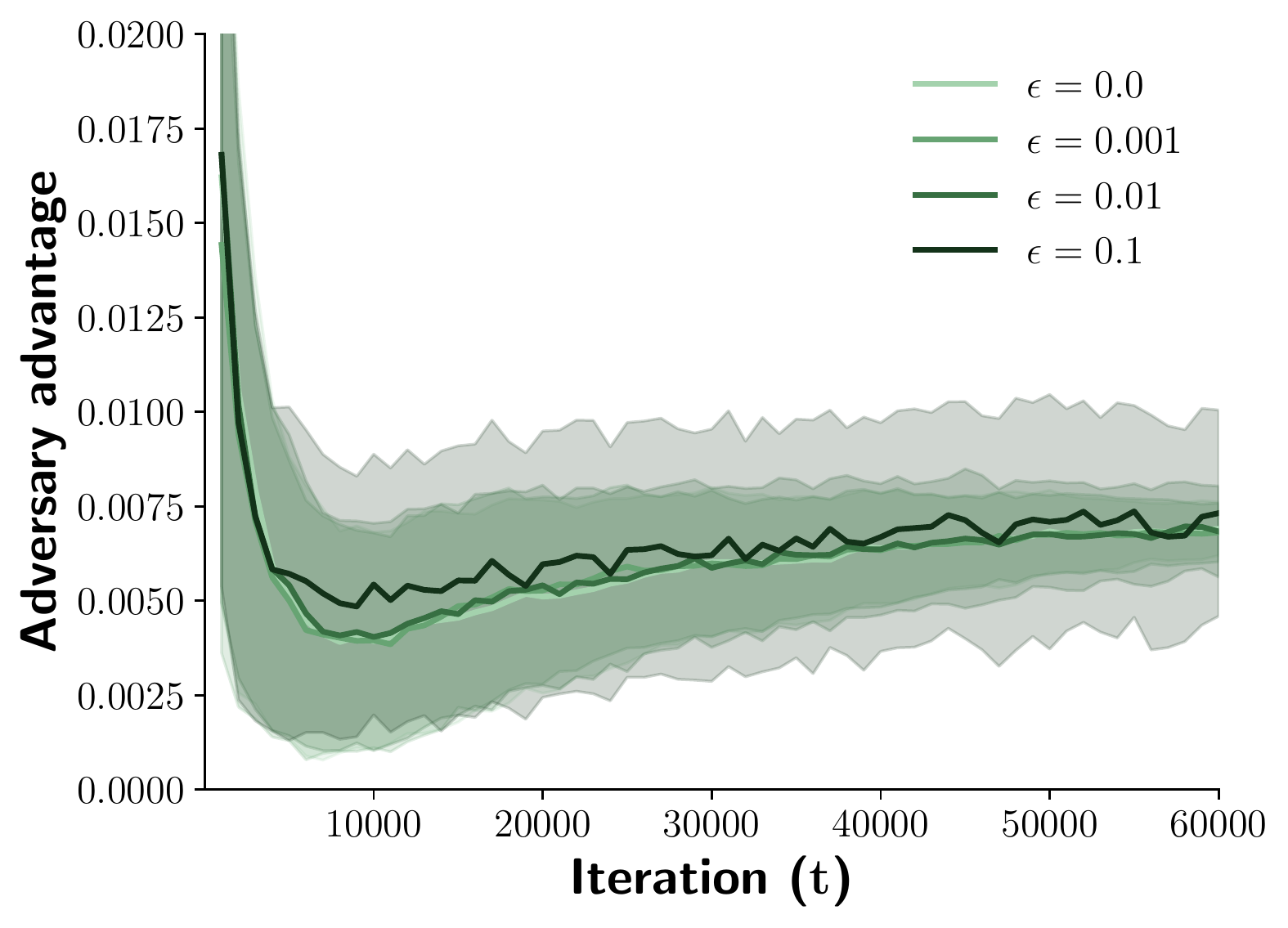}
\caption{Average accuracy (over 200 runs) of
membership-inference attacks~\citep{yeom2018privacy} on
our bandit algorithm at different stages of
training, for four different values of the exploration
parameter, $\epsilon$.}
\label{fig:membership}
\end{figure}

\noindent\textbf{Membership inference attacks.} To empirically measure the
privacy of our contextual-bandit algorithm, we also performed experiments in which
an adversary tries to infer whether or not a sample was part of the training dataset by applying the membership inference
attack of~\citet{yeom2018privacy} on model checkpoints saved at various points during training. The
membership inference attack computes an empirical estimate of the joint
action-reward distribution on the data that was used to train the model and on
a held-out test set, respectively (we use the MNIST test set as held-out set).
We use the resulting empirical distributions, $p_{train}(r | a)$ and $p_{test}(r | a)$, to infer training-data
membership for an example $\bx_{T+1}$. Specifically, we: (1) evaluate the model
on $\bx_{T+1}$, (2) observe the selected arm $a_{T+1}$, (3) receive corresponding reward $r_{T+1}$, and (4) predict
training data membership if $p_{train}(r_{T+1} | a_{T+1}) > p_{test}(r_{T+1} | a_{T+1})$.

Following~\citet{yeom2018privacy}, we measure the \emph{advantage} of the resulting
adversary: the difference in the true positive rate and the false positive rate
in predicting training set membership. The adversary advantage during training
is shown in Figure~\ref{fig:membership} for models trained with different
values of $\epsilon$. The results show that in the early stages of learning,
the adversary has a slight advantage of $1\mbox{--}2\%$, this advantage rapidly
decreases below $0.75\%$ after the learner has observed more training
examples\footnote{For higher values of $\epsilon$, there is more variance in
the model parameters during training, which is reflected in higher variance in
the advantage values.}. The advantage slightly increases in the later stages of training: interestingly, this happens because the model slightly underfits on the MNIST dataset. Overall, the results of our experiments suggest that our contextual
bandit learner is, indeed, maintaining privacy of the context features, $\bx$, well in practice.

\section{Related Work}
\label{sec:related_work}
This study fits into a larger body of work on privacy-preserving machine learning. Prior work has used similar techniques from secure multi-party computation (and homomorphic encryption) for secure evaluation and/or training of deep networks \citep{dowlin2016cryptonets,hynes2018efficient,juvekar2018gazelle,mohassel2017secureml,sadegh2017chameleon,shokri15privacy,wagh2018securenn} and other classification models \citep{decock2019efficient,reich2019privacy}. Other related work has developed secure data-aggregation techniques \citep{bonawitz2017aggregation} for use in federated-learning scenarios \citep{bonawitz2019federated}. To the best of our knowledge, our study is the first to use this family of techniques in the online-learning setting of contextual bandits, using the randomness introduced by exploration mechanisms to obtain a differential-privacy guarantee on the output produced by the learner.

Most closely related to our work are studies on differential private online learning \citep{dwork2019continual,jain2012diff,thakurta2013nearly}. In particular, \citet{mishra2015private} develops UCB and Thompson samplers for (non-contextual) bandits with differential-privacy guarantees based on tree-based aggregation~\citep{chan2019private,dwork2009complexity}. Follow-up work improved differentially private UCB to have better regret bounds~\citep{tossou2016algorithms}. Recent work~\citep{shariff2018diff} also developed a \emph{joint} differentially private version of LinUCB~\citep{li2010contextual}. In contrast to those prior studies, we study a more challenging setting in which the parties that implement the learner may not leak information about their observed contexts, actions, and rewards to each other. Having said that, our algorithm may be improved using the differentially private mechanisms of~\citet{mishra2015private,tossou2016algorithms}. In this study, we opted for the simpler epsilon-greedy mechanism because it can be implemented efficiently on arithmetically shared data. We leave the implementation of differentially private UCB and Thompson samplers in our secure multi-party computation framework to future work.

\section{Summary and Future Work}
\label{sec:discussion}
We presented a privacy-preserving, multi-party contextual bandit algorithm that works correctly in practice and that comes with theoretical guarantees on (differential) privacy of the context features. Although our experimental evaluation of this algorithm demonstrates its effectiveness, several avenues for improvement remain:
\begin{itemize}[leftmargin=*]
\setlength\itemsep{0em}
\item \textbf{Increase numerical stability.} Repeated use of the Sherman-Morrison formula is known to produce cancellation errors that may lead to numerical errors. A numerically stable algorithm would regularly compute the actual matrix inverse to eliminate such errors, and add a diagonal-regularizer to prevent ill-conditioning issues.

\item \textbf{Robustness to disappearing parties.} In many practical settings, parties may temporarily disappear because of system failures~\citep{bonawitz2019federated}. To allow the algorithm to operate in such scenarios, different types of secret sharing (\emph{e.g.}, Shamir sharing~\citep{shamir1979}) may be needed. The contextual bandit itself could learn to be robust to failing parties by employing a kind of ``party dropout'' at training time \citep{srivastava2014dropout}.

\item \textbf{Security under stricter security models.} The current algorithm assumes parties are honest-but-curious, which means that parties do not deviate from the protocol in Algorithm~\ref{alg:bandits}. It is important to note that out privacy guarantees do not hold in stricter security models in which one or more parties operate adversarially or in settings in which the parties collude. Our current algorithm can be extended to provide guarantees under stricter security models: for instance, extending the algorithm to use \emph{message authentication codes} \citep{goldreich2009} would allow the parties to detect attacks in which a minority of the parties behaves adversarially. Unfortunately, such extensions generally increase the computational and communication requirements of the learner.

\item \textbf{Robustness to side-channel attacks.} In practical scenarios, it may be possible to break the privacy of our learner via \emph{side-channels attacks}. For example, there is a delay between taking the action and receiving the reward that may make the learner susceptible to timing attacks \citep{kocher1996timing}: if the distribution of reward delays depends on the action being selected, parties $p \in \mathcal{P}$ may be able to infer the selected action from the observed time delay, counter to our guarantees. A real-world implementation of our algorithm should, therefore, introduce random time delays in the operations performed by parties $p'$ and $p''$ to prevent information leakage.

\item \textbf{Stronger membership-inference attacks.} The membership-inference attacks we considered in this study \citep{yeom2018privacy} are not designed to use the full action-reward sequence as side information in the attack. It may thus be possible to strengthen these membership-inference attacks by using the full action-reward sequence, which may be observed by an external observer of the algorithm. The development of such stronger attacks may help to obtain better empirical insights into the level of privacy provided by our privacy-preserving contextual bandits.

\item \textbf{Scaling to larger problems.} Our algorithm was tested on the task of digit recognition where the action set is small and the feature space was reduced. We leave as follow-on work scaling the algorithm to more realistic problems with larger actions sets, for example, by developing efficient approximations to the $\argmax$ function.
\end{itemize}

 \section*{Acknowledgements}

The authors thank Mark Tygert, Ilya Mironov and Xing Zhou for helpful discussions and
comments on early drafts of this paper.

\bibliographystyle{icml2020}
\bibliography{references.bib}

\pagebreak
\appendix

\section{Secret Sharing}

Our privacy-preserving contextual bandits use two different types of secret sharing: (1) arithmetic secret
sharing~\cite{spdz}; and (2) binary secret sharing~\cite{gmw}. Below, we
describe the secret sharing methods for single values $x$ but they can
trivially be extended to real-valued vectors $\bx$.

\paragraph{Arithmetic secret sharing.} Arithmetic secret sharing is a type of
secret sharing in which the sum of the shares reconstruct the original data
$x$. We refer to the shared representation of $x$ as $[x]$. The shared
representation across parties, $\mathcal{P}$, is given by $[x] = \{\forall p \in \mathcal{P}
: [x]_p\}$, where $[x]_p$ indicates the share of $x$ that party $p \in \mathcal{P}$ has. The
representation has the property that $\sum_{p\in \mathcal P} [x]_p \mod Q \!=\! x$. To
make sure that none of the parties can learn any information about $x$ from
their share $[x]_p$, shares need to be sampled uniformly from a ring of
size $Q$, $\thefield$, and all computations on shares must be performed
modulus $Q$. If $x$ is real-valued, it is encoded to lie in $\thefield$
using the mechanism described in Appendix~\ref{sec:encoding} before it is
encrypted.

To encrypt the unencrypted data $x$, party $p$ that possesses $x$ draws $|\mathcal{P}|-1$
numbers uniformly at random from $\thefield$ and distributes them among the
other parties. Subsequently, party $p$ computes its own share as $[x]_p = x -
\sum_{q \neq p} [x]_q$. Thus all the parties (including party $p$) obtain a
random number that is uniformly distributed over the ring, from which they cannot
infer any information about $x$. To decrypt $[x]$, the parties communicate
their shares and compute $x = \sum_{p \in \mathcal P} [x]_p \mod Q$.

\paragraph{Binary secret sharing.} Binary secret sharing is a special type of
arithmetic secret sharing for binary data in which the ring size $Q = 2$~\citep{gmw}. Because addition modulo two is equivalent to taking an exclusive
OR (XOR) of the bits, this type of sharing is often referred to as XOR secret
sharing. To distinguish binary shares from arithmetic shares, we denote a
binary share of variable $x$ across $P$ parties by $\langle x\rangle = \{
 \forall p \in \mathcal{P}: \langle x\rangle_p \}$. Just as with arithmetic sharing,
 binary secret shares allow for ``linear'' operations on bits without decryption.
 For example, binary sharing allows for the evaluation of any circuit expressed
 as XOR and AND gates. While it is much more efficient to do addition and
 multiplication of integers with arithmetic shares, logical expressions such as
 $\textrm{max}(0, x)$ are more efficient to compute with binary shares. In
 equations, we denote AND by $\otimes$ and XOR by $\oplus$.

To encrypt the unencrypted bit $x$, party $p \in \mathcal{P}$ that possesses $x$ draws
$|\mathcal P|-1$ random bits and distributes those among the other parties. These
form the shares $\forall q \neq p: \langle x \rangle_q$ for $|\mathcal P|-1$ parties.
Subsequently, party $p \in \mathcal{P}$ computes its own share as $\langle x \rangle_p = x
\oplus \left( \bigoplus_{q \neq p} \langle x \rangle_q \right)$. Thus all the
parties (including party $p$) obtain a random bit from which they cannot
infer any information about $x$.

\subsection{Converting Between Secret-Sharing Types}

Contextual bandit algorithms involve both functions that are easier to compute on arithmetic secret shares (\emph{e.g.}, matrix multiplication) and functions that are easier to implement via on binary secret shares (\emph{e.g.}, argmax) using binary circuits. Therefore, we use both types of secret sharing and convert between the two types using the techniques proposed in~\cite{demmler2015aby}.

\noindent\textbf{From $[x]$ to $\langle x\rangle$:} To convert from an arithmetic
share $[x]$ to a binary share $\langle x \rangle$, each party first secretly
shares its arithmetic share with the other parties and then performs addition
of the resulting shares. To construct the binary share $\langle [x]_p \rangle$
of its arithmetic share $[x]_p$, party $p \in \mathcal{P}$: (1) draws $|\mathcal P|-1$ random bit strings
$\langle [x]_p \rangle_q$ and shares those with the other parties and (2)
computes its own binary share $\langle [x]_p \rangle_p = \bigoplus_{q \neq p}
\langle [x]_p \rangle_q$. The parties now each obtained a binary share of $[x]_p$
without having to decrypt $x$. This process is repeated for each party $p \in \mathcal{P}$ to create
binary secret shares of all $|\mathcal{P}|$ arithmetic shares $[x]_p$. Subsequently, the
parties compute $\langle x \rangle = \sum_{p\in \mathcal P} \langle[x]_p\rangle$. The summation is implemented by Ripple-carry adder that can be evaluated in $\log_2 (|\mathcal P| \log_2 Q)$
rounds~\cite{catrina2010improved, damgrd2005unconditionally}.

\noindent\textbf{From $\langle x\rangle$ to $[x]$:} To convert from a binary share
$\langle x \rangle$ to an arithmetic share $[x]$, the parties compute $[x] =
\sum_{b=1}^B 2^b \left[\langle x\rangle^{(b)}\right]$, where $\langle x\rangle^{(b)}$
denotes the $b$-th bit of the binary share $\langle x\rangle$ and $B$ is the total number of bits in the shared secret. To create the arithmetic
share of a bit, $\left[\langle x\rangle^{(b)}\right]$, each party $p \in \mathcal{P}$
draws a number uniformly at random from $\{0, 1\}$ and shares the difference between
their bit and the random number with the other parties. The parties sum all resulting shares to
obtain $\left[\langle x\rangle^{(b)}\right]$.

\subsection{Logical Operations and the Sign~Function}
We rely on binary secret sharing to implement logical operations and the sign function.

\noindent\textbf{XOR and AND.} XOR and AND are addition and multiplication modulo 2
where the numbers belong to the set $\{0, 1\}$ --- they are the binary
operations in $\mathbb{Z}_2$. As a result, the techniques we use for addition
and multiplication of arithmetically shared values (see paper) can be used to implement XOR and AND as well.
Evaluating $\langle x \rangle \oplus  a$ function amounts to one party
$p \in \mathcal{P}$ computing $\langle x \rangle_p \oplus a$, and evaluating $\langle x \rangle
\oplus \langle y \rangle$ amounts to each party $p \in \mathcal{P}$ computing $\langle x
\rangle_p \oplus \langle y \rangle_p$. Similarly, $\langle x \rangle \otimes a$
is evaluated by having each party compute $\langle x \rangle_p \otimes a$. The
AND operation between two private values, $\langle x \rangle \langle y
\rangle$, is implemented akin to the private multiplication protocol using Beaver
triples.

\noindent\textbf{Sign function.} We express the $\textrm{sign}$ function on an arithmetically shared value as $[x \ge 0]$. Using
this expression, the sign function can be implemented by first converting the
arithmetic share, $[x]$, to a binary share, $\langle x \rangle$, using the
conversion procedure described above. Subsequently, we obtain the most
significant bit, $\langle x\rangle^{(MSB)}$, and convert it back to an
arithmetic share to obtain $[x \ge 0]$.

\section{Fixed-Precision Encoding}
\label{sec:encoding}

Contextual bandit algorithms generally use real-valued parameters and data. Therefore, we need to encode the real-valued numbers as integers before we can arithmetically share them.
We do so by multiplying $x \in \mathbb{R}$ with a large scaling factor $B$ and rounding to
the nearest integer: $\hat{x} = \lfloor B x \rceil$, where $B = 2^L$ for some
precision parameter, $L$. We decode an encoded value, $\hat{x}$, by
computing $x = \hat{x} / B$. Encoding real-valued numbers this way incurs a precision loss
that is inversely proportional to $L$.

Since we scale by a factor $B \!=\! 2^L$ to encode floating-point numbers, we
must scale down by a factor $2^L$ after every multiplication. We do this using
the public division protocol described in Appendix~\ref{sec:pub_division}.

\section{Public Division}
\label{sec:pub_division}

\begin{algorithm}[t]
\caption{Private computation of the number of wraps in an arithmetically shared value.}
\label{alg:wraps}
\begin{algorithmic}
\STATE {\bfseries Input:}\\
- Arithmetically secret-shared value, $[x]$.\\
- A store of random numbers and the number of wraps in those numbers generated off-line.\\
\STATE {\bfseries Output:} Arithmetic sharing of the number of wraps in $[x]$, denoted by $[\theta_x]$.\\
Select next random number, $[r]$, and the number of wraps $[\theta_r]$ in $[r]$. \\
$[z] \leftarrow [x] + [r]$ \\
\FOR{$p \in \mathcal P$}
    \STATE
	Party $p$ computes $\beta_{x_pr_p}$ from $[x]_p$ and $[r]_p$ such that $[z]_p = [x]_p + [r]_p - \beta_{x_pr_p} Q$.
\ENDFOR \\
Parties construct arithmetic sharing, $[\beta_{xr}] = \{\forall p \in \mathcal P : [\beta]_p\}$. \\
Parties decrypt $[z]$ to obtain $z$ (note that $z$ contains no information about $x$). \\
Parties compute number of wraps in $z$, $\theta_z$. \\
Parties compute $[\eta_{xr}] \leftarrow z < [r]$. \\
Parties compute $[\theta_x] \leftarrow \theta_z + [\beta_{xr}] - [\theta_r] - [\eta_{xr}]$ (number of wraps in $[x]$). \\
\end{algorithmic}
\end{algorithm}

A simple method to divide an arithmetically shared value, $[x]$, by a public value, $\ell$, would simply divide the share of each party by $\ell$). However, such a method can produce incorrect results when the sum of shares ``wraps around'' the ring size, $Q$. Defining $\theta_x$ to be the number of wraps such that $x = \sum_{p \in \mathcal P} [x]_p - \theta_x Q$, indeed, we observe that:
\[
    \frac{x}{\ell} = \sum_{p \in \mathcal P} \frac{[x]_p}{\ell} - \frac{\theta_x}{\ell} Q \ne \sum_{p \in \mathcal P} \frac{[x]_p}{\ell} - \theta_x Q.
\]
Therefore, the simple division method fails when $\theta_x \ne 0$, which happens with probability $P(\theta_x \neq 0) = \frac{x}{Q}$ in the two-party case. Many prior MPC implementations
specialize to the $|\mathcal P| =2$ case and rely on this probability being
negligible~\cite{mohassel2017secureml, sadegh2017chameleon, wagh2018securenn}.
However, when $|\mathcal P|>2$ the probability of failure grows rapidly and we must account
for the number of wraps, $\theta_x$.

We do so by privately computing a secret share of the number of wraps in $x$, $[\theta_x]$, using Algorithm~\ref{alg:wraps}. We use $[\theta_x]$ to compute the correct value of the division by $\ell$:
\[
    \frac{x}{\ell}  = [z] - [\theta_x] \frac{Q}{\ell} \quad \text{where} \quad
    [z] = \left\{\forall p \in \mathcal P : \frac{[x]_p}{\ell}\right\}.
\]

In practice, it can be difficult to compute $[\eta_{xr}]$ in Algorithm~\ref{alg:wraps} (line 8).
We note that $\eta_{xr}$ has a fixed probability of being non-zero, irrespective of whether the number of parties is two or larger, \emph{i.e.}, regardless of the number
of parties $P(\eta_{xr} \neq 0) = \frac{x}{Q}$. In practice, we therefore skip the
computation of $[\eta_{xr}]$ and simply set $\eta_{xr}=0$. This implies that incorrect results can be produced
by our algorithm with small probability. For example, when we multiply two real-values, $\hat
x$ and $\hat y$, the result will be encoded as $B^2 \hat x \hat y$ which has
probability $\frac{B^2\hat x \hat y}{Q}$ of producing an error. This
probability can be reduced by increasing $Q$ or reducing the precision
parameter, $B$.

\begin{figure}[h!]
\centering
\includegraphics[width=\linewidth]{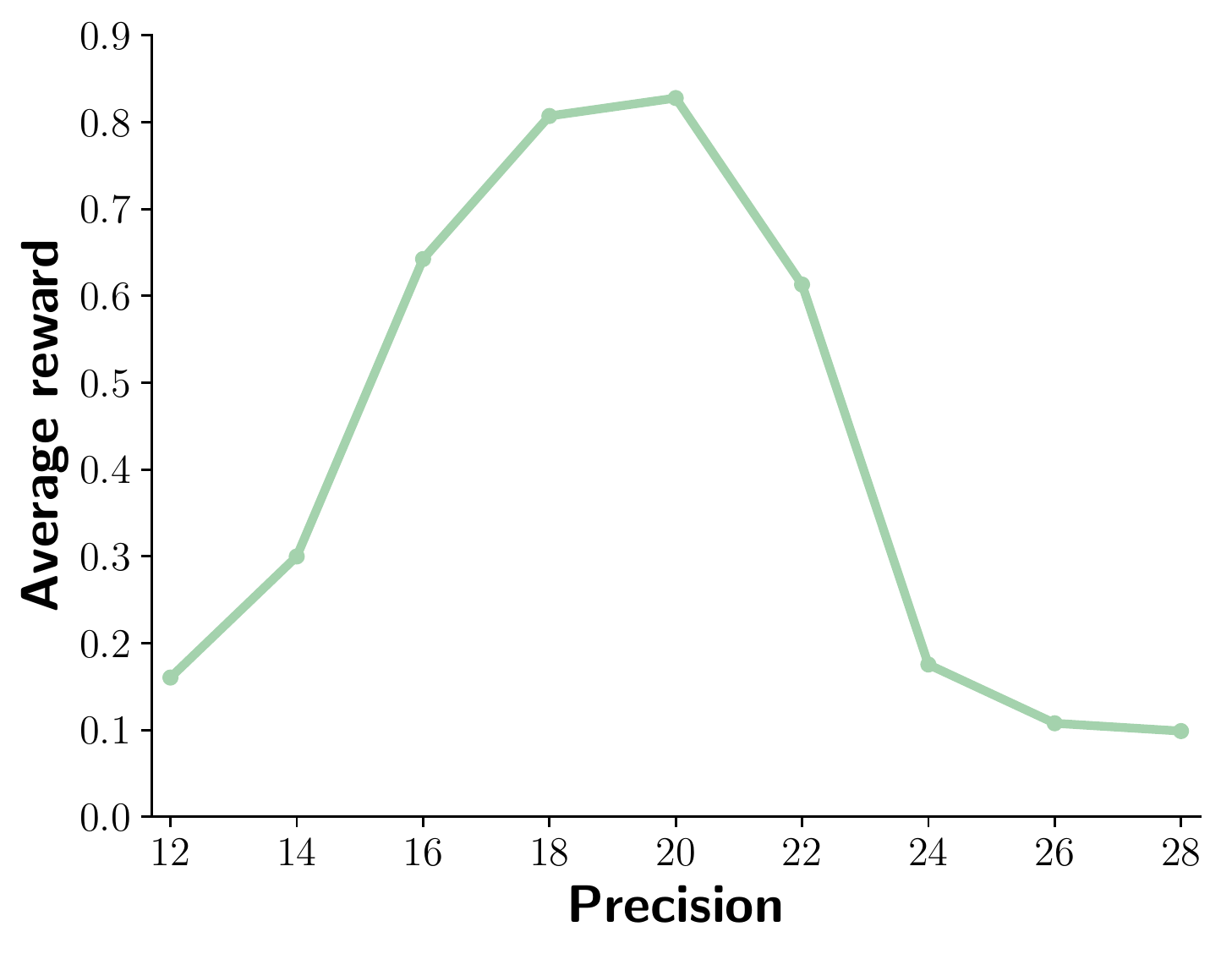}
\caption{Average reward as a function of the bits of precision, $L$, used to
encode floating-point values into $\thefield$. We multiply values by $2^L$ to encode them as fixed-point
integers in $\thefield$.}
\label{fig:reward_precision}
\end{figure}

\begin{figure}[h!]
\centering
\includegraphics[width=\linewidth]{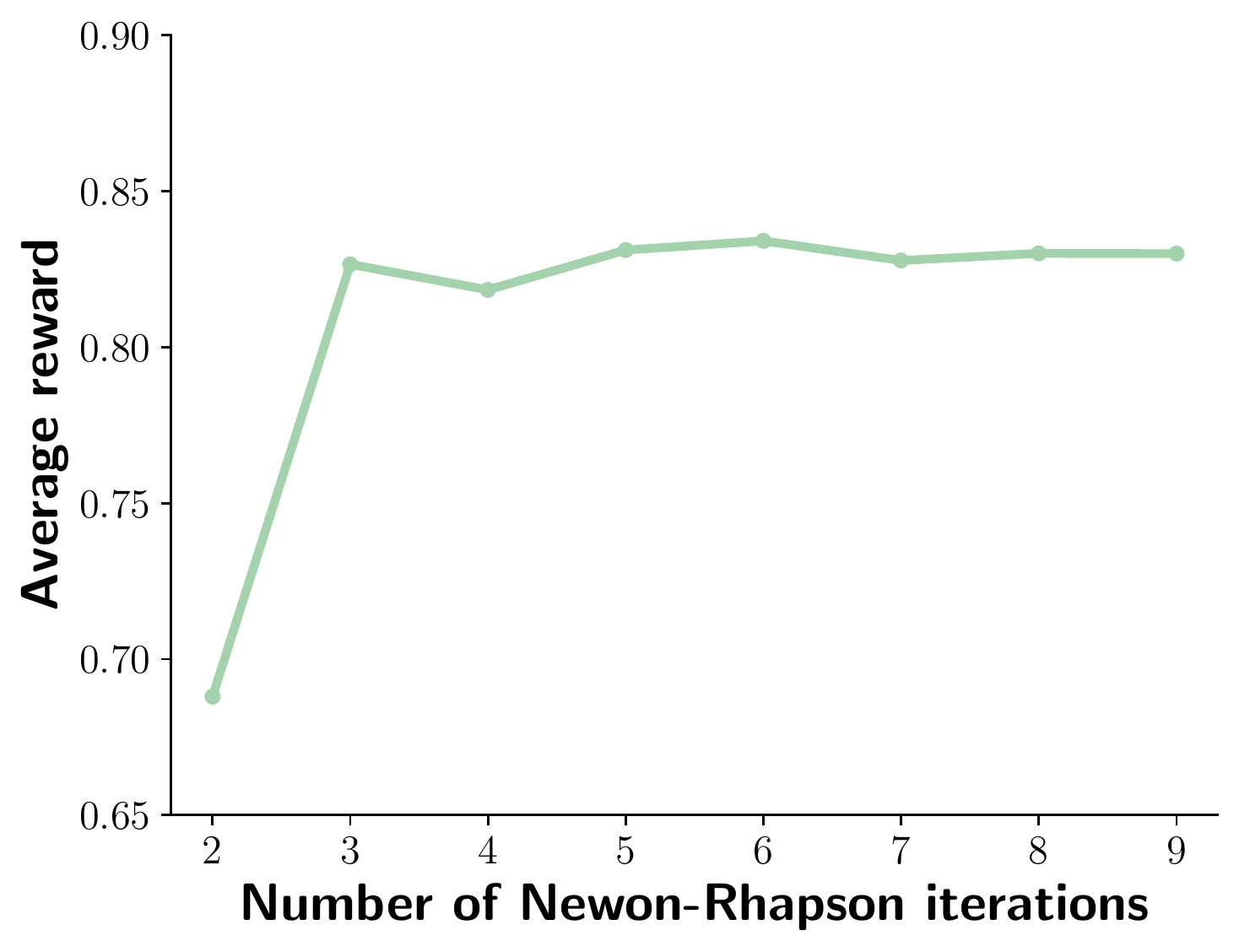}
\caption{Average reward as a function of the number of Newton-Rhapson
iterations used for privately computing the reciprocal function (see Section~3 in the main paper for details).}
\label{fig:reward_nrits}
\end{figure}

\section{Numerical Precision}
\label{sec:numerical_precision}
In addition to the experiments presented in the paper, we also performed experiments to measure the impact of varying the precision in the fixed-point encoding and the numerical approximations.
Figure~\ref{fig:reward_precision} shows the average reward as a function of the
bits of precision used in encoding of floating-point values. The optimal
precision is $20$ bits with a sharp drop in reward obtained below $18$ and above $22$ bits. The drop below $18$ bits is due precision loss causing numerical instability. Algorithm~1 in the main paper is susceptible to three forms of
numerical instability: (1) ill-conditioning due to relying on the normal
equations to solve the least-squares problem, (2) degeneracies in the matrix
$\bW_a^{-1}$ which can become singular or non-positive-definite, and (3)
cancellation errors due to use of the Sherman-Morrison update. The drop in observed reward when using more than $22$ bits of precision is due to wrap-around
errors that arise because we do not correctly compute $\eta_{xr}$ (see Section~\ref{sec:encoding}). This causes public divisions to fail catastrophically with higher probability, impeding the accuracy of the learning algorithm.

Figure~\ref{fig:reward_nrits} shows how the average reward changes as a
function of the number of Newton-Rhapson iterations used to approximate the
private reciprocal. The results reveal that three iterations suffice in our
experiments. We note that the domain of the private reciprocal in
Algorithm~1 in the main paper is $[1, c]$. In all our experiments, we observed empirically that $c < 10$.
For use cases that require a larger range of values, $[1, c]$, more iterations and a different initial
value may be needed to ensure rapid convergence.

\end{document}